\documentclass[runningheads]{llncs}
\usepackage{graphicx}
\usepackage{xcolor}
\usepackage{hyperref}
\definecolor{blued}{RGB}{70,197,221}
\definecolor{pearYe}{HTML}{FFB733}
\definecolor{pearOne}{HTML}{2C3E50}
\definecolor{pearTwo}{HTML}{A9CF54}
\definecolor{pearTwoT}{HTML}{C2895B}
\definecolor{pearThree}{HTML}{E74C3C}
\colorlet{titleTh}{pearOne}
\colorlet{bull}{pearTwo}
\definecolor{pearcomp}{HTML}{B97E29}
\definecolor{pearFour}{HTML}{588F27}
\definecolor{pearFith}{HTML}{ECF0F1}
\definecolor{pearDark}{HTML}{2980B9}
\definecolor{pearDarker}{HTML}{1D2DEC}
\hypersetup{
	colorlinks,
	citecolor=pearDark,
	linkcolor=pearThree,
	urlcolor=pearDarker}
\usepackage{amsmath,amssymb,amsbsy} 

\usepackage{booktabs}
\usepackage{colortbl}
\usepackage[width=122mm,left=12mm,paperwidth=146mm,height=193mm,top=12mm,paperheight=217mm]{geometry}

\usepackage[lofdepth,lotdepth]{subfig}
\usepackage[export]{adjustbox}

\definecolor{graphicbackground}{rgb}{0.96,0.96,0.8}
\definecolor{rouge1}{RGB}{226,0,38}  
\definecolor{orange1}{RGB}{243,154,38}  
\definecolor{jaune}{RGB}{254,205,27}  
\definecolor{blanc}{RGB}{255,255,255} 
\definecolor{rouge2}{RGB}{230,68,57}  
\definecolor{orange2}{RGB}{236,117,40}  
\definecolor{taupe}{RGB}{134,113,127} 
\definecolor{gris}{RGB}{91,94,111} 
\definecolor{bleu1}{RGB}{38,109,131} 
\definecolor{bleu2}{RGB}{28,50,114} 
\definecolor{vert1}{RGB}{133,146,66} 
\definecolor{vert3}{RGB}{20,200,66} 
\definecolor{vert2}{RGB}{157,193,7} 
\definecolor{darkyellow}{RGB}{233,165,0}  
\definecolor{lightgray}{rgb}{0.9,0.9,0.9}
\definecolor{darkgray}{rgb}{0.6,0.6,0.6}
\definecolor{babyblue}{rgb}{0.54, 0.81, 0.94}
\definecolor{citrine}{rgb}{0.89, 0.82, 0.04}
\definecolor{misogreen}{rgb}{0.25,0.6,0.0}

\newcommand*\diff{\mathop{}\!\mathrm{d}}

\newcommand{\pa}[1]{\left(#1\right)}

\newcommand{\transpose}{^\mathsf{\scriptscriptstyle T}}

\newcommand{\eps}{\varepsilon}
\renewcommand{\epsilon}{\varepsilon}
\renewcommand{\hat}{\widehat}
\renewcommand{\tilde}{\widetilde}

\newcommand{\bSigma}{\mathbf{\Sigma}}
\newcommand{\bGamma}{\mathbf{\Gamma}}

\newcommand{\nothere}[1]{}

\usepackage{xspace}

\title{Compressing the Input for CNNs with the First-Order Scattering Transform} 

\titlerunning{Compressing the Input for CNNs
with the First-Order Scattering Transform}

\authorrunning{E.\,Oyallon, E.\,Belilovsky, S.\,Zagoruyko, M.\,Valko}

\author{Edouard Oyallon,$^{\!\!1,4,5}$ Eugene Belilovsky,$\!^2$ Sergey Zagoruyko,$\!^3$ Michal Valko$^4$}

\institute{$^1$CentraleSupelec, Universit\'e Paris-Saclay\\
$^2$MILA, University of Montreal\\
$^3$WILLOW -- Inria Paris,
$^4$SequeL -- Inria Lille,
$^5$GALEN -- Inria Saclay
\\
}

\begin{document}

\maketitle



\begin{abstract}

We study the \emph{first-order} scattering transform as a candidate for reducing the signal processed by a \emph{convolutional neural network} (CNN). We show theoretical and empirical evidence that in the case of natural images and sufficiently small translation invariance, this transform preserves most of the signal information needed for classification while \emph{substantially reducing} the spatial resolution and total signal size.  We demonstrate that cascading a CNN with this representation performs on par with ImageNet classification models, commonly used in downstream tasks, such as the ResNet-50. We subsequently apply our trained hybrid ImageNet model as a base model on a detection system, which has typically larger image inputs. On Pascal VOC and COCO detection tasks we demonstrate  improvements in the inference speed and training memory consumption compared to models trained directly on the input image. 

\end{abstract}

\section{Introduction}
Convolutional neural networks (CNNs) for supervised vision tasks  learn often from raw images \cite{lecun2010convolutional} that could be arbitrarily large.
Effective reduction of the spatial dimension and total signal size for CNN processing is difficult. One way is to learn this dimensionality reduction during the training of a supervised CNN. Indeed, the very first layers of standard CNNs play often this role and reduce the spatial resolution of an image via \emph{pooling} or \emph{stride operators}. Yet, they generally maintain the input layer sizes and even increase it by expanding the number of channels. These pooling functions can correspond to a linear pooling such as wavelet pooling~\cite{williams2018wavelet}, a spectral pooling~\cite{rippel2015spectral}, an average pooling, or a non-linear pooling such as $\ell^2$-pooling \cite{le2013building}, or max-pooling. For example, the two first layers of an AlexNet \cite{krizhevsky2012imagenet}, a  VGG ~\cite{simonyan2014very} or a ResNet \cite{he2015deep} reduce the resolution respectively by $2^3$, $2^1$, and $2^2$, while the dimensionality of the layer  is increased by a factor $1.2$, $5.3$, and $1.3$ respectively. This spatial size reduction is important for computational reasons because the complexity of convolutions is quadratic in spatial size while being linear in the number of channels. This suggests that reducing the input size to subsequent CNN layers calls for a careful design. In this work, we (a) analyze a generic method that, \emph{without learning}, reduces  \emph{input size} as well as \emph{resolution} and (b) show that it \emph{retains enough information} and structure that permits applying a CNN to obtain competitive performance on classification and detection.

Natural images have a lot of redundancy, that can be exploited by finding a frame to obtain a sparse representation \cite{forsyth2002computer,mallat1999wavelet}. For example, a wavelet transform of piece-wise smooth signals (e.g., natural images) leads to a multi-scale and sparse representation \cite{mallat1992singularity}. This fact can be used for a compression algorithm~\cite{skodras2001jpeg}. Since in this case the most of the information corresponds to just a few wavelet coefficients, a transform coding can be applied to select them and finally quantize the signal, which is consequently a more compact representation. Yet, this leads to variable signal size and thus this method is not amenable for CNNs that require a constant-size input. Another approach is to select a subset of these coefficients, which would be a linear projection. Yet, a linear projection would imply an unavoidable loss of significant discriminative information which is not desirable for vision applications. Thus, we propose to use a \emph{non-linear} operator to reduce the signal size and we justify such construction.



Prior work has proposed to input predefined features into CNNs or neural networks. For example, \cite{perronnin2015fisher} proposed to apply a deep neural network on Fisher vectors. This approach relies on the extraction of overlapping descriptors, such as  SIFT, at irregular spatial locations and thus does not permit a fixed size output. Moreover, the features used in these models increase the signal size. In~\cite{fujieda2017wavelet}, wavelets representations are combined at different layer stages, similarly to DenseNet~\cite{huang2017densely}.  \cite{levinskis2013convolutional} proposes to apply a 2D Haar transform that leads to subsampled representation by a factor of $2^1$ but is limited to this resolution. Concurrent to our work, \cite{gueguen2018faster} proposed to train CNNs on top of raw DCT to improve inference speed by reducing the spatial resolution, yet this transformation is orthogonal and thus preserves the input size. Moreover, \cite{oyallon2017scaling} proposes to input \emph{second-order}  scattering coefficients to a CNN, that are named \emph{hybrid scattering networks},  which lead to a competitive performance on datasets such as ImageNet. The scattering transform is a \emph{non-linear} operator based on a cascade of wavelet transforms and modulus non-linearity which are spatially averaged. This leads to a reduction in the spatial resolution of the signal. However, although the second-order scattering representation is more discriminative, it produces a larger signal than the original input size.

In this work, we also input predefined features into CNNs, \emph{but} with the explicit goal of an initial stage producing a compressed representation that is still amenable to processing by a CNN.  
In particular, we show that the first-order scattering representation is a natural candidate for several vision tasks. This descriptor leads to high accuracy on large-scale classification and detection while it can be computed much faster than its second-order counterpart because it requires fewer convolutions. As explained in \cite{bruna2013invariant}, this descriptor is similar to SIFT and DAISY descriptors that have been used as feature extractors in many classical image classification and detection systems \cite{sanchez2013image,morel2009asift}. In this paper, we show that in the case of hybrid networks \cite{oyallon2017scaling,perronnin2015fisher}, using the \emph{first-order scattering only} can have favorable properties with respect to the second-order ones and possibly higher-order ones.

The core of our paper it the analysis and justification of the combination of first-order scattering and CNNs. We support it both with theoretical and numerical arguments. 
In Section~\ref{fos}, we justify that first-order scattering with small-scale invariance reduces the spatial resolution and signal while preserving important attributes. First, we motivate the first-order scattering from a dimensionality reduction view in Section~\ref{d}. Then, in Section \ref{il}, we illustrate the negligible loss of information via a good reconstruction of synthetic signals and natural images using only a first-order scattering. Next, in Section~\ref{numeric} we present our experiments\footnote{code available at \texttt{\url{https://github.com/edouardoyallon/pyscatlight}}} on challenging datasets.  We demonstrate competitive performance with ImageNet models commonly used in transfer learning in Section~\ref{classif}. In Section~\ref{detection} we show on COCO and Pascal VOC detection tasks that these base networks can lead to improvements in terms of inference speed and memory consumption versus accuracy.

\section{First-order scattering}
\label{fos}
In this section, we motivate the construction of a first-order \emph{scattering transform} from a compression perspective. Indeed, a scattering transform is traditionally built as a representation that preserves high-frequency information, while building stable invariants w.r.t.\,translations and deformations. While using the same tools, we adopt a rather different take. We show theoretically and numerically that a first-order \emph{scattering transform}  builds limited invariance to translation, reduces the size of an input signal, preserves most of the information needed to discriminate and reconstruct a  natural image. Note also that this representation is able to discriminate \emph{spatial} and \emph{frequency variations} of natural images. In this section, we deal with  \emph{G\'abor wavelets}~\cite{olshausen1996emergence}  since their analysis is simpler, while for the experiments we will use modified G\' abor wavelets, namely \emph{Morlet wavelets}~\cite{oyallon2017scaling} for the sake of comparison. We show that the first-order scattering transform does not lose significant signal characteristics of natural images, by providing reconstruction examples obtained via a mean-square error minimization. In particular, we demonstrate this property on \emph{Gaussian blobs} as a simplified proxy for natural images. 
\subsection{A reduction of the spatial resolution}
\label{d}
\subsubsection{Definition}A scattering first-order transform~\cite{mallat2012group} is defined from a \emph{mother wavelet}~$\psi$ and a \emph{low-pass filter} $\phi$. An input signal $x$ is filtered by a collection of dilated band-pass wavelets obtained from~$\psi$, followed by a \emph{modulus} and finally averaged by a \emph{dilation} of~$\phi$. The wavelets we chose decompose the signal in a basis in which transient structure of a signal is represented more compactly. 
We describe the construction of each filter and justify the necessity of each operator. First, let us fix an integer $J$ that specifies the window length of the low-pass filter. For the sake of simplicity, we consider G\'abor filters~\cite{olshausen1996emergence}.  These filters provide a good localization tradeoff between \emph{frequency} and \emph{space planes}, due to Heisenberg uncertainty principle \cite{mallat1999wavelet}. Thus, having
\begin{equation*}
\kappa(\omega)\triangleq e^{-2\sigma_0^2\Vert\omega\Vert^2}
\end{equation*}
for a fixed bandwidth $\sigma_0$ and a slant $s$ that discriminates angles, we  set for $\omega=(\omega_1,\omega_2)$,
\begin{equation*}
\hat \psi(\omega)\triangleq\kappa\pa{\pa{\omega_1,\frac{\omega_2}{s}}-\omega_0}\quad\text{and}\quad\hat \phi(\omega)\triangleq\kappa(\omega).
\end{equation*}
The frequency plane (and in particular the image frequency circle of radius $\pi$) needs to be covered by the support of the filters to avoid an information loss.  This issue is solved by the action of the Euclidean group on $\psi$ via rotation $r_{-\theta}$ and dilation by $j\leq J$,
\begin{equation*}
\psi_{j,\theta}(u)=\frac{1}{2^{2j}}\psi\pa{r_{-\theta}\frac{u}{2^j}}\quad\text{and} \quad\phi_J(u)=\frac{1}{2^{2J}}\phi\pa{\frac{u}{2^J}}\!\cdot
\end{equation*}
In this case, each wavelet $\psi_{j,\theta}$ has a bandwidth of $1/(2^j\sigma_0)$ and its central frequency is $2^jr_{-\theta}\omega_0$. If a filter has a compact support in the frequency domain, then due to Nyquist principle, we can reduce the spatial sampling of the resulting convolution. We do this approximation in the case of G\'abor filters. As we shall see, this localization in frequency is also fundamental because it permits to obtain a \emph{smooth envelope}. The parameters $j\leq J$ and $\theta\in \Theta$ are discretized and $\sigma_0$ is adjusted such that a wavelet transform preserves all the energy of $\hat x$, characterized by
\begin{equation*}
\exists \epsilon_0\geq0,\forall \omega, \Vert\omega\Vert < \pi:\  1-\epsilon_0 \leq \!\!\!\sum_{j\leq J,\theta\in \Theta}\big|\hat \psi_{j,\theta}(\omega)\big|^2+\big|\hat \phi_J(\omega)\big|^2\leq 1+\epsilon_0.
\end{equation*}
\noindent
As a result, the transform is bi-Lipschitz and the magnitude of $\epsilon_0$ determines the conditioning of the wavelet transform. An ideal setting is $\epsilon_0=0$, for which the transform is an isometry which gives a one-to-one mapping of the signal while preserving its $\ell^2$-norm. Applying a convolution with these wavelets followed by a modulus removes the phase of a signal and thus should lead to a loss of information.

In fact, \cite{mallat2015phase} proves that it is possible to reconstruct a signal from the modulus of its wavelet transform up to a global translation with \emph{Cauchy wavelets}. Furthermore, there exists an algorithm of reconstruction \cite{waldspurger2015phase}, with stability guarantees and extension to other class of wavelets. Consequently, the modulus of a wavelet transform does not lead to a significant loss if applied appropriately. Additionally, \cite{mallat2012group} demonstrates that this representation is stable to deformations, which permits building invariants to deformations, convenient in many vision applications. We now explain how the dimensionality reduction occurs.

\noindent
The scattering first-order transform $S$~\cite{mallat2012group} parametrized by $J$ is\footnote{in the following, we omit the dependence w.r.t.\,the scale $J$} defined as
\begin{equation*}
Sx(u)=\left\{\left|x\star \psi_{j,\theta}\right|\star \phi_J(2^Ju),x\star \phi_J(2^Ju)\right\}_{\theta\in\Theta, j\leq J}.
\end{equation*}
The low-pass filter $\phi_J$ builds a transformation that is locally invariant to translation up to $2^J$. Therefore, it reduces the spatial sampling of the signal by a factor of $2^J$. This also means that when discretized image of length $N$ represented by $N^2$ coefficients, is filtered by the low-pass filter $\phi_J$, the signal is represented by $N^2/2^{2J}$ coefficients. Consequently, the number of coefficients used to represent $Sx$ is
\begin{equation*}
\pa{1+|\Theta| J}\frac{N^2}{2^{2J}}\cdot
\end{equation*}

 In our case, we use $|\Theta|=8$, because it permits obtaining a good covering of the frequency plane, and thus, the input signal $x$ is compressed via $Sx$ if $J\geq 3$. The low-pass filtering implies a necessary loss of information because it discards some high-frequency structure and only retains low frequencies which are more invariant to translation. It is fundamental to evaluate the quality of this compressed representation in order to validate that enough information is available for supervised classifier such as CNNs, which is what we do next.

\subsubsection{Preserving signal information via modulus}We evaluate the loss of information due to the low-pass filtering, which captures signal attributes located in the low-frequency domain. Notice that there would be no loss of information if the Fourier transform of the wavelet-modulus representation was located in a compact domain included in the bandwidth of $\phi_J$. Unfortunately, this property is not guaranteed in practice.

Nonetheless, G\'abor wavelets are \emph{approximately analytic} which implies that when convolved with a signal $x$, the resulting envelope is smoother  \cite{krajsek2007unified,mallat1999wavelet,soulard2012ondelettes,mallat2012group,delprat1992asymptotic}. A smooth envelope of the signal implies that a significant part of its energy  can
be captured and preserved by a low-pass filter~\cite{mallat2012group}. Furthermore, under limited assumptions of point-wise regularity on $x$, if the signal does not vanish, it is possible to quantify this smoothness, as done in \cite{delprat1992asymptotic}.   Informally, for a translation~$x_a(u)\triangleq x(u-a)$ by $a$ of $x$, it means that if $\Vert a\Vert \ll1$, then we imply that
\[|x_a\star \psi|(u)\approx |x\star\psi|(u).\]
Here, we simply give some explicit constant w.r.t.\,the stability to translation, that we relate to the envelope of $\psi$. Indeed, the G\'abor filter $\psi$ concentrates its energy around a central frequency~$\omega_0$,
\[\exists\eta_{0}>0,\exists\omega_{0},\epsilon\geq0,  \forall \omega,\quad \Vert\omega-\omega_{0}\Vert_2 >\eta_0 \implies |\hat\psi(\omega)|\leq\epsilon.\] 

\noindent

First-order  scattering  incorporates more information if the modulus operator has smoothed the signal. To this end, we characterize the stability w.r.t.\,translations in the case of G\'abor wavelets. In particular, we provide the following Lipschitz bound w.r.t.\,translations.
\begin{proposition}\label{prop1}
For any signal $x\in \ell^2\!\!,$
\[\Vert x_a \star \psi-e^{-i\omega_0\transpose a}x\star\psi\Vert\leq \Vert x \Vert\pa{\Vert\eta_0\Vert\Vert a\Vert+\tilde \epsilon\pa{\lVert a \rVert}}\!,\]
where $\tilde \epsilon$ is a term of the order of $\eps$.\end{proposition}
\begin{proof}
Observe that
\begin{align*}
\lVert x_a \star \psi-e^{-i\omega_0\transpose a}&x\star\psi\rVert^2
\!=\!\int\!\Big|\!\pa{e^{-i\omega\transpose a}-e^{-i\omega_{0}\transpose a}}\hat{\psi}(\omega)\hat x(\omega)\Big|^{2}\!\!\!\diff\omega\text{\footnotesize\ \emph{via Parseval identity}} \\
&\leq	4\epsilon^2\Vert x\Vert ^2+\int_{\Vert\omega-\omega_{0}\Vert<\eta_{0}}\big|(e^{-i\omega\transpose a}-e^{-i\omega_{0}\transpose a})\hat{\psi}(\omega)\big|^{2}\big|\hat x(\omega)\big|^2\diff\omega.\\
\intertext{(note that $x\mapsto e^{ix}$ is 1-Lipschitz, thus we apply the Cauchy-Schwartz inequality)}
&\leq	\Vert x\Vert ^2\pa{4\epsilon^2+\Vert a\Vert^{2}\eta_{0}^2}.
\end{align*}
 Taking the square root finishes the proof.
\end{proof}

We note that this inequality is near-optimal, for G\'abor wavelets, if $x(u)=\delta_0$ is a Dirac in 0, then $|e^{i\omega_0\transpose a}\psi(a)-\psi(0)|\sim \Vert x \Vert \Vert a \Vert \eta_0$. Observe that dilating the mother wavelet $\psi$ to $\psi_j$ is equivalent to dilating the bandwidth $\eta_0$ to $2^{-j}\eta_0$. Following the reasoning, low-frequency  G\' abor wavelets are more likely to be invariant to a translation. 

Proposition~\ref{prop1} characterizes the Lipschitz stability w.r.t.\,translations and  indicates that the more localized a G\'abor wavelet is, the more translation-stable is the resulting signal. This way, we justify G\'abor wavelets as a great candidate for a wavelet transform with a smooth modulus with limited assumptions on $x$. 

Note that using only small-bandwidth G\'abor wavelets instead of dilated ones should be avoided because it would lead to significantly more filters. Furthermore, \cite{mallat2012group} shows that those filters will be more unstable to deformations, such as dilation, which is not desirable for vision applications.  

Despite the stability to translation, there is no guarantee that the first-order scattering preserves the complete energy of the signal. The next section  characterizes this energy loss  via an image model based on Gaussian blobs and a reconstruction algorithm for natural images.

\begin{figure}[h]
\centering
\subfloat[\textbf{middle}: $\text{PSNR} \approx 26dB,$ right: $\text{PSNR} \approx 20dB$]{
\centering
\includegraphics[height=3.8cm]{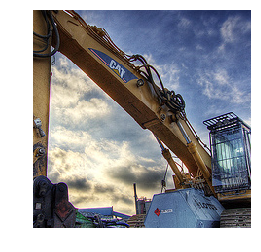}
\includegraphics[height=3.8cm]{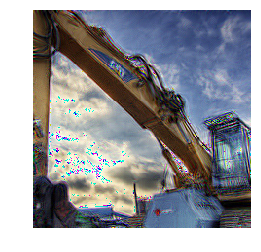}
\includegraphics[height=3.8cm]{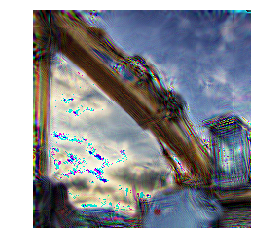}
}\\
\subfloat[\textbf{middle}: $\text{PSNR} \approx 23dB,$ right: $\text{PSNR} \approx 19dB$]{
\centering
\includegraphics[height=3.8cm]{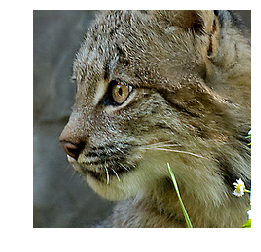}
\includegraphics[height=3.8cm]{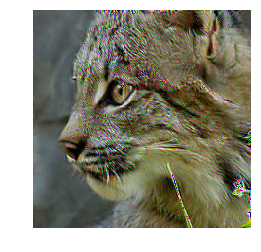}
\includegraphics[height=3.8cm]{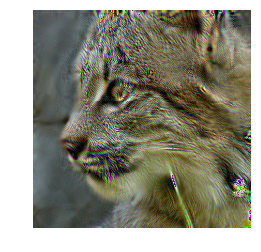}
}

\caption{Reconstructed images from first-order scattering coefficients, $J=3,4$ and their PSNR. Color channels can be slightly translated leading to artifacts. (\textbf{left}) original image $x$ (\textbf{middle}) reconstruction $\tilde x_3$ from $Sx$, $J=3$ (\textbf{right}) reconstruction $\tilde x_4$ from $Sx$, $J=4$. This demonstrates that even complex images can be reconstructed for $J=3$, while dividing the spatial resolution by $2^3$.}
\label{fig:reconstruction}
\end{figure}

\subsection{Information loss}
\label{il}
We now characterize the information loss for natural images in two ways. First, we perform an empirical reconstruction of an image from its first-order scattering coefficients, as done for the second order in \cite{bruna2013audio,bruna2013scattering} and observe that for natural images,  we can indeed obtain an effective reconstruction of the first-order scattering. This is a strong indication that the relevant signal information is preserved. Second, we consider a generic signal model and show that for relatively low-scale factors $J$, the reconstruction is practically achieved. 
\subsubsection{Reconstruction} Following \cite{bruna2013audio,bruna2013scattering}, we propose to reconstruct an input image~$x$ from its first-order scattering $Sx$ coefficient of scales $J$, via a $\ell^2$-norm minimization
 \begin{equation}
 \tilde x_J = \inf_y \Vert Sx-Sy\Vert. \label{min}
 \end{equation}
We use a gradient descent as all the operators are weakly differentiable and we analyze the reconstructed signal. Figure~\ref{fig:reconstruction} compares the reconstruction of a natural image with the first-order scattering for the scales $J=3$ and $J=4$. In our experiments, we optimize for this reconstruction with ADAM with an initial learning rate of 10 during $10^3$ iterations, reducing by 10 the learning rate every $2\times10^2$ iterations. We measure the reconstruction error of $\tilde x_J$ from an original image $x$ in terms of relative error, defined as \[\text{err}_J(x)=\frac{\Vert S\tilde x_J-Sx\Vert}{\Vert Sx\Vert}\cdot\]
In other words, we evaluate how close the scattering representation of an image is to its reconstruction. We stop the optimization procedure as soon as  we get $\text{err}_J(x)\sim 2\times10^{-3}$. In the case $J=3$, observe that the important and high-frequency structure of the signals, as well as their spatial localization, are preserved.  On the contrary, when $J=4$, the fine-scale structure is neither well reconstructed nor correctly located, which tends to indicate that $J\geq 4$ might not be a good-scale candidate for  $S$. We now characterize this loss more precisely on a model based on blobs.

\subsubsection{Gaussian blob model}Explicit computation of scattering coefficients for general signals is difficult because a modulus is a non-linear operator that usually leads to non-analytic expressions.  Therefore, we consider a simplified class of signals \cite{lindeberg1998feature} for which computations are exact and analytical. For a symmetric matrix $\bSigma$, we consider the unnormalized signal

\begin{equation*}
\hat{x}_\bSigma(\omega)\triangleq e^{-\omega\transpose\bSigma \omega}.
\end{equation*}
    Figure \ref{fig:reconstruction_blob} shows several signals belonging to  this class. Such signals correspond to blobs or lines as on Figure \ref{fig:reconstruction_blob}, which are frequent in natural images~\cite{lowe2004distinctive}.  We apply our reconstruction algorithm and we explain why the reconstructions is challenging.

\begin{figure}[h]
\centering
\vspace{2em}
\subfloat[ellipse]{
\centering
\includegraphics[height=2cm,trim={2cm 3cm 3cm 3.2cm},clip]{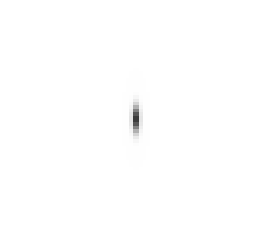}
\includegraphics[height=2cm,trim={2cm 3cm 3cm 3.2cm},clip]{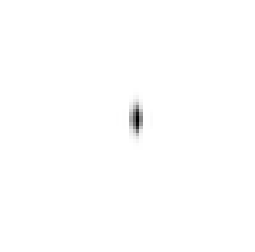}
\includegraphics[height=2cm,trim={2cm 3cm 3cm 3.2cm},clip]{ellipse_J_4}
}\\
\subfloat[small blob]{
\centering
\includegraphics[height=2.6cm,trim={3.4cm 3.4cm 3.4cm 3.4cm},clip]{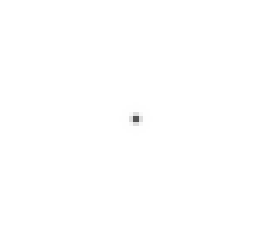}
\includegraphics[height=2.6cm,trim={3.4cm 3.4cm 3.4cm 3.4cm},clip]{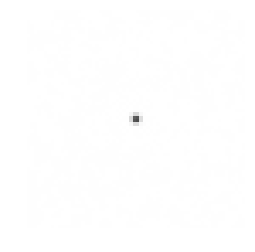}
\includegraphics[height=2.6cm,trim={3.4cm 3.4cm 3.4cm 3.4cm},clip]{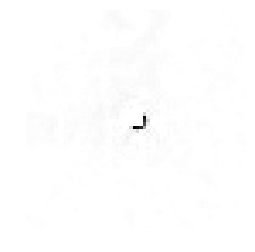}
}
\\
\subfloat[line]{
\centering
\includegraphics[height=2.4cm,trim={3.4cm 3.4cm 3.4cm 3.4cm},clip]{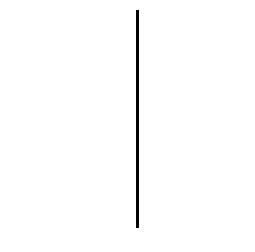}
\includegraphics[height=2.4cm,trim={3.4cm 3.4cm 3.4cm 3.4cm},clip]{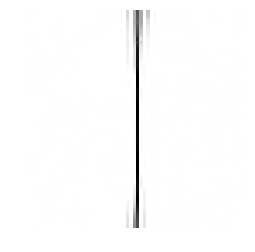}
\includegraphics[height=2.4cm,trim={3.4cm 3.4cm 3.4cm 3.4cm},clip]{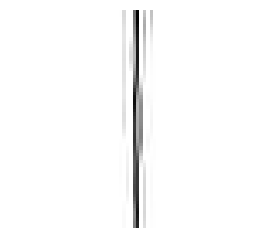}
}
\caption{Reconstruction of different signals of type $x_\bSigma$. (\textbf{left}) original image (\textbf{middle}) reconstruction via for $J=3$ (\textbf{right}) reconstruction for $J=4$}
\label{fig:reconstruction_blob}
\end{figure}
In particular, we prove the following proposition that is derived from convolutions between Gaussians and permits to compute their first-order scattering coefficients.
Intuitively, this proposition says that for a particular class of signals, we can get their exact reconstruction from their first-order scattering coefficients.
Note that for large values of $J$, the reconstruction is numerically infeasible.

\begin{proposition}For any symmetric $\bSigma$, $j$, and $\theta$, 
\[|x_{\bSigma}\star \psi_{j,\theta}|(u)\propto \big(x_\bSigma\star|\psi_{j,\theta}|\big)(u).\]
\end{proposition}
\begin{proof}
Without loss of generality, we prove the result for $\hat \psi(\omega)\triangleq e^{-\Vert \bGamma \omega-b\Vert^2}$, where $\bGamma$ is invertible and $b\in\mathbb{R}^2$. Then, $\hat{|\psi|}(\omega)\propto e^{-\Vert \bGamma \omega \Vert^2}$.  Let $\Delta(u) \triangleq x_\bSigma \star\psi(u)$.  Then by definition,
\begin{align*}
\hat \Delta(\omega) &\propto  e^{-\omega\transpose(\bSigma+\bGamma\transpose\bGamma)\omega+2\omega\transpose\bGamma b}. 
\end{align*}
As $\bGamma\transpose\bGamma\succ0$, we can set $\tilde b \triangleq \big(\bSigma+\bGamma\transpose\bGamma\big)^{-1}\bGamma b$. Then, the result comes from an inverse Fourier transform applied to
\[\hat\Delta(\omega) \propto e^{-(\omega-\tilde b)\transpose\pa{\bSigma+\bGamma\transpose\bGamma}\pa{\omega-\tilde b}}.\]
\end{proof}
Therefore, the first-order scattering coefficients are given by
\[Sx_{\bSigma} \propto \{x_{\bSigma}\star (|\psi_{j,\theta}|\star\phi_J),x_{\bSigma}\star\phi_J\}_{\theta \in \Theta,j\leq J}.\]

A na\"ive inversion of the first-order scattering coefficients would be an inversion of the convolution with $|\psi_{j,\theta}|\star\phi_J$  which is unfortunately poorly conditioned for large values of $J$ since this filter is a low-pass one. However, solving the optimization \eqref{min} leads to a different solution due to the presence of the modulus during the gradient computation. For $J\leq 3$, it is possible to recover the original signal as shown in Figure~\ref{fig:reconstruction_blob}. Nevertheless, there is a  lack of spatial localization for $J\geq 4$ due to the averaging $\phi_J$, that we observe during our reconstruction experiment. This confirms our choice of $J=3$ for the remainder   of the paper.

\section{Numerical experiments}
\label{numeric}
We perform numerical experiments using first-order scattering output as the input to a CNN.
Our experiments aim to both validate that the first-order scattering can preserve the key signal information and highlight the practical importance of it. In particular, we find that we obtain performance close to some of the state-of-the-art systems \textit{while improving inference speed} and memory consumption during training
by light years..

\vfil
\subsection{ImageNet classification experiments}
\label{classif}
We first describe our image classification experiments on the challenging ImageNet dataset. Each of our experiments is performed using standard hyperparameters \textit{without a specific adaptation} to our hybrid architecture. Extensive architecture search for the first-order scattering input is not performed and we believe that these results can be improved with resources matching the architectures and hyperparameters developed for natural images.

ImageNet ILSVRC2012 is a challenging dataset for classification. It consists of $1$k classes, $1.2$M large colored images for training and $400$k images for testing. We demonstrate that our representation does not lose significant information for classification by obtaining a competitive performance on ImageNet. We follow a standard procedure training procedures \cite{zagoruyko2016wide,oyallon2017scaling,he2015deep}. Specifically, we applied standard data augmentation  and crop input images to a size of $224^2$. The first order scattering then further reduces this to a size of $28\times28$. We trained our CNNs by stochastic gradient descent (SGD) with the momentum of 0.9, weight decay of $10^{-4}$ and batch size of 256 images, trained for $90$ epochs. We reduce the learning rate by $0.1$ every 30 epochs. At test time, we rescale the images to $256^2$ and crop an image of size $224^2$.

To construct our scattering hybrid networks, we stay close to an original CNN reference model. In particular, we build our models out of the ResNets \cite{he2015deep} and WideResNets \cite{zagoruyko2016wide} models. A typical ResNet consists of an initial layer followed by $K=4$ so-called layer groups that in turn consist of $[n_1,\dots,n_K]$ residual blocks, where $n_i$ specifies the number of blocks in each layer group. Furthermore, the width in each blocks is a constant and equal to $[w_1,\dots,w_K]$. Similarly to \cite{oyallon2017scaling}, an initial convolutional layer is applied to increase the number of channels from $3\times(1+8J)=75$ to $w_1$. A stride of 2 is applied at the initial layer of the blocks $k\geq 2$, to reduce the spatial resolution. Each of the residual blocks  contains two convolutional operators, except when a stride of 2 is applied in order to replace the identity mapping, in which case there are three convolutional operators, as done in \cite{he2015deep}. In the following, we refer to ScatResNet-$L$ as the architecture with $L$ convolutional operators. As discussed we used $J=3$, as done in \cite{oyallon2017scaling}.

In our first experiment, we aim to directly compare to the results of \cite{oyallon2017scaling} which use the second-order scattering. Thus we use the same structure that applies $K=2$ layer groups on the scattering input instead of the typical 4. This architecture was called the ScatResNet-10 \cite{oyallon2017scaling}, and has $[2,2]$ layers of width $[256,512]$. The number of parameters is about $12$M in both cases. Notice that the number of parameters varies only since the initial number of input channels change. Table~\ref{table:ImageNet} reports similar accuracy for order 1 and order 2 scattering, which indicates that if enough data is available and there is a  small invariance to translation $J$, then for natural image classification, the order 2 does not provide significantly more information that can be exploited by a CNN.

Now we demonstrate that the scattering first-order transform continues to scale further when applying more sophisticated networks. Note that this would not have been possible with a second-order  scattering in a reasonable time. In our case, we avoid computing many convolutions. Scaling to these modern networks permits us to  apply the scattering in the subsequent section to common computer vision tasks that require a base network from ImageNet, and where the smaller input size leads to gains in speed and memory.

The models we construct are the ScatResNet-50, based on the ResNet50 architecture, and the WideScatResNet-50-2 based on the wide ResNet that expands the channel width and leads to competitive performance \cite{zagoruyko2016wide}. Since the scattering input starts at a much lower resolution, we bypass the first group of the typical ResNet, which normally consists of $K=4$ layer groups and reduce the number of groups to $K=3$. A typical ResNet50 has 16 residual blocks distributed among the 4 layer groups. We maintain the same number of total residual blocks and thereby layers as in the ResNet50, redistributing them among the three groups using $[5,8,3]$ blocks. As in their non-scattering analogue we apply bottleneck blocks \cite{he2015deep}. The width of the blocks for ScatResNet-50 and WideScatResNet-50-2 are $[128,256,512]$ and $[256,512,1024]$, which matches the widths of groups $2$ through $4$ of their non-scattering counterparts.   
\setlength{\tabcolsep}{4pt}
\begin{table}
\begin{center}
\caption{Accuracy on ImageNet. Note that scattering based models have input sizes of $28\times28\times75$ while the normal ImageNet models are trained on $224 \times 224 \times 3$.}
\label{table:ImageNet}
\ \\
\ \\
\begin{tabular}{>{}r>{\columncolor{white!10}}c>{\columncolor{white!10}}c>{\columncolor{white!10}}r}
\toprule
\cellcolor{white!20}\textbf{Architecture} & \cellcolor{white!20}\textbf{Top 1} & \cellcolor{white!20}\textbf{Top 5} & \cellcolor{white!20}\textbf{\#params}\\
\midrule
Order 1,2 + ScatResNet-10 \cite{oyallon2017scaling}&68.7&88.6 &12.8M\\
 Order 1 + ScatResNet-10  & 67.7 &87.7&11.4M \\
\midrule
Order 1 + ScatResNet-50   &74.5 & 92.0&27.8M\\
Order 1 + WideScatResNet-50-2 &76.2 &92.8&107.2M\\
ResNet-50 ({\tt pytorch})& 76.1 &92.9&25.6M\\
ResNet-101({\tt pytorch}) &77.4 & 93.6 & 45.4M\\
\midrule
VGG-16 \cite{simonyan2014very} & 68.5&88.7&138M\\
ResNet-50 \cite{he2015deep}& 75.3 &92.2&25.6M\\
ResNet-101 \cite{he2015deep}&76.4 & 92.9 & 45.4M\\
WideResNet50-2 \cite{zagoruyko2016wide} &77.9 &94.0 &68.9M \\
ResNet-152 \cite{he2015deep}&77.0&93.3&60.2M\\
\bottomrule
\end{tabular}
\end{center}
\vspace{-2em}
\end{table}

\noindent
Table~\ref{table:ImageNet} indicates that the performance obtained by those architectures can be competitive with their respective reference architectures for classification. We compare to the reference models trained using the same procedures as ours.\footnote{\tt \url{http://pytorch.org/docs/0.3.0/torchvision/models.html}} We additionally compare to published results of these models and several related ones. We evaluate the memory and speed of the ScatResNet-50 model and compare it to the reference models ResNet-50 and the next biggest ResNet model ResNet-101 in the first two rows of Table ~\ref{table:speedMem}. Our comparisons are done on a single GPU. As in \cite{gueguen2018faster,torfason2018towards}, we evaluate the inference time of the CNN from the encoding. For memory, we consider memory usage during training as we believe the scattering models are useful for  training with fewer resources. We find that our scattering model has favorable properties in memory and speed usage compared to its non-scattering analogues. In fact, as the next step, we demonstrate large improvements in accuracy, speed, and memory on detection tasks using the ScatResNet-50 network, which indicates that ScatResNet-50 features are also generic for detection.

\begin{table}
\begin{center}
\caption{Speed and memory consumption for ImageNet classification sizes (224x224) and detection scale 800px. We compare the inference speed of the learned CNN between the different models and for the detection models the inference speed of feature extraction. To evaluate memory we determine the maximum batch size possible for training on a single GPU. We use a single 11GB Ti 1080 GPU for all comparisons. }
\label{table:speedMem}
\ \\
\ \\
\begin{tabular}{rcc|cc}
\toprule
 &\multicolumn{2}{c}{\textbf{Classification Models}} & \multicolumn{2}{c}{\textbf{Detection Models}}\\\cmidrule{2-5} 
  & Speed   & Max im. & Speed & Max im.\\
\textbf{Architecture}&  (64 images) & ImageNet  & (4 images) & Coco\\\midrule
Order 1 + ScatResNet-50   &0.072 &175&0.073&9\\
ResNet-50 & 0.095 &120 &0.104 &7\\
ResNet-101 & 0.158 &\phantom{0}70&  0.182&2\\
\bottomrule
\end{tabular}
\end{center}
\vspace{-2em}
\end{table} 
\subsection{Detection experiments}
\label{detection}
Finally, we apply our hybrid architectures to detection. We base our experiments and hyperparameters on those indicated by the Faster-RCNN implementation of  \cite{jjfaster2rcnn} without any specific adaptation to the dataset. We consider both the VOC07 and COCO and  adopt  the ScatResNet-50 network as the basis of our model. We shared the output of the second layer across a region proposal network and a detection network, which are kept fixed.  The receptive field of each output neuron corresponds to $16^2$, which is similar to \cite{girshick2015fast,he2015deep,ren2017faster}. The next layers will be fine-tuned for the detection tasks and fed to classification and box-regression layers, as in \cite{girshick2015fast}, and a \emph{region proposal network} as done in \cite{ren2017faster}. Similarly to \cite{he2015deep,ren2017faster}, we fixed all the batch normalization~\cite{ioffe2015batch}  layers, including the running means and biases.
\subsubsection{Pascal VOC07}

\begin{table}
\begin{center}
\caption{Mean average precision on Pascal VOC7 dataset. First-order scattering permits outperforms the related models.}
\label{table:Pascal}
\ \\
\ \\
\begin{tabular}{>{}r>{\columncolor{white!10}}c}
\toprule
\cellcolor{white!20}\textbf{Architecture} & \cellcolor{white!20}\textbf{mAP}  \\
\midrule

Faster-RCNN Order 1 + ScatResNet-50 (ours) &73.3\\
Faster-RCNN ResNet-50 (ours) & 70.5 \\
Faster-RCNN ResNet-101 (ours) & 72.5 \\
\midrule
Faster-RCNN VGG-16  \cite{jjfaster2rcnn}&70.2\\

\bottomrule
\end{tabular}
\end{center}
\vspace{-2em}
\end{table}
\setlength{\tabcolsep}{1.4pt}

Pascal VOC2007 \cite{everingham2010pascal} consists of $10$k images split equally for training (``train+val'') and testing, with about $25$k annotations. We chose  the same hyperparameters as used in \cite{jjfaster2rcnn}. We used  an initial learning rate of $10^{-3}$  that we dropped by $10$ in epoch 5 and we report the accuracy of the epoch 6 in Table~\ref{table:Pascal} on the test set. During training, the images are flipped and rescaled with a ratio between 0.5 and 2, such that the smaller size is 600px as~\cite{he2015deep,ren2017faster}.

\noindent The training procedures  used for detection often vary substantially. This includes batch size, weight decay for different parameters, and the number of training epochs among others. Due to this inconsistency, we train our own baseline models. We use the trained base networks for ResNet-50 and ResNet-101 provided as part of the {\tt torchvision} package in {\tt pytorch} \cite{paszke2017automatic} and train the detection models in exactly the same way as described above for ScatResNet50, ResNet-50, and ResNet-101. Table \ref{table:Pascal} reports a comparison of our ScatResNet model and the ResNet50 and ResNet101 model on this task. The results clearly show that our architecture and base network leads to a substantially better performance in terms of the mAP. On this particular dataset, perhaps due to its smaller size, we find the hybrid model can outperform even models with substantially stronger base networks \cite{oyallon2017scaling} i.e the performance of ScatResNet-50 is above that of the ResNet101 based model. In the second two rows of Table~\ref{table:speedMem}, we show the memory and speed of the different models. The inference speed of the base network feature extractor is shown and for memory, we show the maximum batch size that one can train with. The tradeoff in mAP vs. speed and mAP vs. memory consumption here clearly favors the scattering based models. We now consider a larger scale version of this task on the COCO dataset.  

\subsubsection{COCO}
\setlength{\tabcolsep}{4pt}
\begin{table}
\begin{center}
\caption{Mean average precision on COCO 2015 minival. Our method obtains competitive performance with respect to popular methods.}
\ \\
\ \\
\label{table:COCO}

\begin{tabular}{>{}r>{\columncolor{white!10}}c}
\toprule
\cellcolor{white!20}\textbf{Architecture} & \cellcolor{white!20}\textbf{mAP}  \\
\midrule

Faster-RCNN Order 1 + ScatResNet-50  &32.2\\
Faster-RCNN ResNet-50 (ours) & 31.0 \\
Faster-RCNN ResNet-101 (ours) & 34.5\\
\midrule
Faster-RCNN VGG-16 \cite{jjfaster2rcnn} &29.2\\
Detectron \cite{he2017mask}&41.8\\
\bottomrule
\end{tabular}
\end{center}
\vspace{-2em}
\end{table}
\setlength{\tabcolsep}{1.4pt}
We likewise deploy the ScatResNet-50 on the COCO dataset \cite{lin2014microsoft}. This detection dataset is more difficult than than PASCAL VOC07. It has $120$k images, out of which we use 115k for training and 5k for validation (minival), with 80 different categories. We again follow the implementation of \cite{jjfaster2rcnn} and their training and evaluation protocol. Specifically, we train the Faster-RCNN networks for 6 epochs with an initial learning rate of $8\times10^{-3}$, multiplying by a factor $0.1$ at epoch 4. We use a minimal size of 800px, and similar scale augmentation w.r.t.\,Pascal VOC07. We use a batch size of 8  on 4 GPUs and train again all 3 models ScatResNet-50, ResNet-50, and ResNet-101. At test time, we restrict the maximum size to be 1200px as in \cite{jjfaster2rcnn} to permit an evaluation on a single GPU.

Table \ref{table:COCO} reports the mAP of our model compared to its non-hybrid counterparts. This score is computed via the standard averaged over IoU thresholds $[0.5,0.95]$. Our architecture accuracy falls between the one of a ResNet-50 and a ResNet-101.   Observing Table ~\ref{table:speedMem} the tradeoff in mAP vs. speed and mAP vs. memory consumption here still favors the scattering based models. The results indicate that scattering based models can be favorable even in sophisticated \textit{near}-state-of-the-art models. We encourage future work on combining scattering based models with the \textit{most}-state-of-the-art  architectures and pipelines.

\section{Conclusion}
We consider the problem of compressing an input image while retaining the information and structure necessary to allow a typical CNN to be applied. To the best of our knowledge, this problem has not been directly tackled with an effective solution. We motivate the use of the \textit{first-order scattering} as a candidate for performing the \emph{signal reduction}. We first refine several theoretical results regarding the stability with respect to translation of the first-order scattering. This motivates the use of G\'abor wavelets that capture many signal attributes. We then show both on an analytical model and  experimentally that reconstruction is possible. We perform experiments on challenging image classification and detection datasets ImageNet and COCO, showing that CNNs approaching the state-of-the-art performance can be built on top of the first-order scattering. This work opens the way to a research on transformations that build compressed input representations. Finally, we incite research on families of wavelets that could increase the  resolution reduction and on determining whether our result generalizes to other classes of signals.

\vspace{0.3cm}
{\footnotesize
\noindent\textbf{Acknowledgements}
E.\,Oyallon was supported by a GPU donation from NVIDIA and partially supported by a  grant from the DPEI 
of Inria (AAR 2017POD057) for the collaboration with CWI. S.\,Zagoruyko was supported by the DGA RAPID project DRAAF. The research presented was also supported by European CHIST-ERA project DELTA, French Ministry of
Higher Education and Research, Nord-Pas-de-Calais Regional Council,
Inria and Otto-von-Guericke-Universit\"at Magdeburg associated-team north-European project Allocate, and French National Research Agency projects ExTra-Learn (n.ANR-14-CE24-0010-01) and BoB (n.ANR-16-CE23-0003).}
\vfil

%
%
%
%
%
%
%
%


\bibliographystyle{splncs}
\bibliography{egbib}
\end{document}